\documentclass[letterpaper,10pt]{article}
\usepackage[utf8]{inputenc}
\usepackage{url}
\usepackage{graphicx}
\usepackage{tikz}
\usepackage{paralist}
\usepackage{amsthm}
\usepackage{amsmath}

\title{Exploring the Use of Shatter for \allsat{} Through Ramsey-Type Problems}
\author{David E. Narváez\\Golisano College of Computing and Information Sciences\\Rochester Institute of Technology\\Rochester, NY, USA 14623\\den9562@rit.edu}

\newcommand{\sat}{\textsc{SAT}}
\newcommand{\allsat}{\textsc{AllSAT}}
\newcommand{\maxsat}{\textsc{MaxSAT}}
\newcommand{\sharpsat}{\textsc{\#SAT}}
\newcommand{\toolname}[1]{#1}
\DeclareMathOperator{\image}{Im}

\let\originalleft\left
\let\originalright\right
\renewcommand{\left}{\mathopen{}\mathclose\bgroup\originalleft}
\renewcommand{\right}{\aftergroup\egroup\originalright}
 
\newtheorem{theorem}{Theorem}
\newtheorem{lemma}{Lemma}
\newtheorem{definition}{Definition}

\begin{document}

 \maketitle

\begin{abstract}
In the context of \sat{} solvers, \toolname{Shatter} is a popular tool for symmetry breaking on CNF formulas. Nevertheless, little has been said about its use in the context of \allsat{} problems: problems where we are interested in listing all the models of a Boolean formula. \allsat{} has gained much popularity in recent years due to its many applications in domains like model checking, data mining, etc. One example of a particularly transparent application of \allsat{} to other fields of computer science is computational Ramsey theory. In this paper we study the effect of incorporating \toolname{Shatter} to the workflow of using Boolean formulas to generate all possible edge colorings of a graph avoiding prescribed monochromatic subgraphs. Generating complete sets of colorings is an important building block in computational Ramsey theory. We identify two drawbacks in the naïve use of \toolname{Shatter} to break the symmetries of Boolean formulas encoding Ramsey-type problems for graphs: a ``blow-up'' in the number of models and the generation of incomplete sets of colorings. The issues presented in this work are not intended to discourage the use of \toolname{Shatter} as a preprocessing tool for \allsat{} problems in combinatorial computing but to help researchers properly use this tool by avoiding these potential pitfalls. To this end, we provide strategies and additional tools to cope with the negative effects of using \toolname{Shatter} for \allsat{}. While the specific application addressed in this paper is that of Ramsey-type problems, the analysis we carry out applies to many other areas in which highly-symmetrical Boolean formulas arise and we wish to find all of their models.
\end{abstract}

\section{Introduction}

The \allsat{} problem is a variant of the \sat{} problem where we are interested in finding all the models (satisfying assignments) of the input formula. Applications of \allsat{} to fields like model checking \cite{yadgar:09,grumberg:04,mcmillan:02} and data mining \cite{jabbour:13} have drawn attention to \allsat{} solvers. The survey by Toda and Soh \cite{toda:16} summarizes the state-of-the-art in techniques used for solving \allsat{} problems. Many ideas from \sat{} solvers are applicable to \allsat{} solvers with few adaptations. For example, the basic idea behind \allsat{} solvers based on blocking-clauses \cite{mcmillan:02} is to iteratively find a model using a conventional \sat{} solver and then add a clause that blocks that model, until no models are found. This intuitive idea serves as a base for other more advanced approaches to \allsat{} solvers \cite{yu:06}. A technique that has proved effective in \sat{} solvers and is of particular importance for \allsat{} problems is symmetry breaking in Boolean formulas \cite{sakallah:09}. Much research has been devoted to develop methods that prevent \sat{} solvers from exploring redundant search spaces in search for a satisfying assignment. The survey by Walsh \cite{walsh:12} provides a good overview of the current approaches and techniques used to deal with this problem. Extensions of symmetry breaking for \sat{} have been studied for different related problems like \maxsat{} \cite{kopp:15} and \sharpsat{} \cite{kopp:16}, and in the context of Answer Set Programming \cite{devriendt:16b}. \toolname{Shatter} \cite{aloul:06} is a tool that generates \emph{symmetry-breaking clauses} as a preprocessing step to solve Boolean formulas in order to simplify the search space for conventional \sat{} solvers. \toolname{Shatter} has become a popular preprocessing tool since it can be used on any CNF formula encoded in the popular DIMACS format and it has the desirable property that the size of the symmetry-breaking clauses it adds to the formula is linear in the number of variables of the original formula.

To study the effect of \toolname{Shatter}'s symmetry breaking approach in a clean application where symmetries can be formally defined and studied, we look at symmetry breaking for Ramsey-type problems in combinatorial computing. Ramsey-type problems are those related to conditions by which it is impossible to avoid prescribed substructures when partitioning the set of elements of a discrete object. Computational approaches to solve these kinds of problems have recently leveraged advances on \sat{} solvers \cite{heule:17}. In this paper we focus on graph 2-colorings, which are partitions of the set of edges of a graph into two sets. We say that a graph $F$ \emph{arrows} the pair of graphs $\left(G,H\right)$, written $F\rightarrow\left(G,H\right)$, when any 2-coloring of the edges of a graph $F$ yields a monochromatic $G$ in the first color or a monochromatic $H$ in the second color.

It is straightforward to see how 2-color arrowing problems can be encoded into Boolean satisfiability problems. In the case of finite Ramsey numbers, a description of this encoding appears in Zhang's chapter in the Handbook of Satisfiability \cite{zhang:09}. Boolean formulas arising from this encoding have a prominent place in complexity theory because they provide hard examples for resolution. The complexity of resolution of such formulas was studied by Krishnamurthy \cite{krishnamurthy:81}, Krajíček \cite{krajivcek:01}, and Pudlák \cite{pudlak:12}, among others.

On the one hand, determining the satisfiability of a Boolean formula encoding $F\not\rightarrow\left(G,H\right)$ (i.e., solving the \sat{} problem for this formula) can answer the question of whether $F$ arrows the pair $\left(G,H\right)$. On the other hand, finding all the assignments that satisfy the formula (i.e., solving the \allsat{} problem for this formula) can be used to generate the complete set of colorings of the edges of $F$ witnessing $F\not\rightarrow\left(G,H\right)$. Generating these sets of colorings for one parameter $\left(G,H\right)$ is often used as a building block towards finding exact values for Ramsey numbers under different parameters. The ``gluing method'' used to establish $R\left(4,5\right)=25$ \cite{mckay:95}, which is still found embedded in more recent ideas \cite{codish:16}, is an example of such an application. Using \allsat{} solvers for this purpose offloads the combinatorial search to standard tools that are being actively developed \cite{toda:16}, eliminating the need to craft specialized code. A shortcoming of generating these families through the method of encoding the negation of the arrowing property into a Boolean formula is that an \allsat{} solver may generate many colorings that are essentially equivalent among them as it lists all possible models of the formula. The need to generate models that are distinct under some notion of equivalence is not unique to formulations of the arrowing predicate as a Boolean formula. Many symmetry breaking techniques have been developed for specific applications of \sat{}. A great example of a symmetry breaking technique specifically tailored for graph search is the use of \emph{canonizing sets} for small graph searches \cite{itzhakov:16}. Unfortunately, the need to embed these techniques into the Boolean encoding of the non-arrowing property may neglect the advantage of using general-purpose \allsat{} solvers. To avoid this issue, one may be interested in using off-the-shelf symmetry breaking software that is domain independent.

Given Shatter's popularity, it may seem as a good tool to tackle the generation of irredundant sets of colorings for Ramsey-type problems. Nevertheless, in this paper we identify and discuss some drawbacks of using this approach without caution. One of these drawbacks stems from the fact that Shatter's main goal is to avoid adding symmetry-breaking clauses whose number of literals is quadratic in the number of variables while making it easier to find \emph{a} (not \emph{every}) satisfying assignment of the original formula should one exist. This is done at the expense of preserving the number of ``unique'' satisfying assignments of the original formula. Another drawback stems from the fact that \toolname{Shatter} works directly at the Boolean formula level to detect symmetries, so it has no access to domain-specific symmetries that are not carried over by the encoding. The issues presented in this work are not intended to discourage the use of \toolname{Shatter} as a preprocessing tool for \allsat{} problems in combinatorial computing but to help researchers properly use this tool by avoiding these potential pitfalls. To this end, we provide strategies and additional tools to cope with the negative effects of using \toolname{Shatter} for \allsat{}.

The rest of this paper is organized as follows. In Section~\ref{s:background} we review \toolname{Shatter}'s approach to symmetry breaking and give a brief background on Ramsey-type problems. In Sections~\ref{s:numberofmodels} and \ref{s:incompletesets} we address the problems of the number of satisfying assignments of the resulting formulas after preprocessing with \toolname{Shatter} and the generation of incomplete sets of colorings, respectively. Finally, Section~\ref{s:conclusions} adds some concluding remarks and directions for future work.

\section{Background and Definitions}
\label{s:background}

\toolname{Shatter} is a domain-independent preprocessing tool \cite{aloul:06} which is available online\footnote{\url{http://www.aloul.net/Tools/shatter/}} and works on Boolean formulas expressed in the standard DIMACS format, making it very convenient to integrate in any existing workflow. A fundamental construction in \toolname{Shatter}'s formulation of symmetry breaking predicates is a graph that encodes the relationship between clauses and literals of the input formula \cite{aloul:03}. This graph is defined as follows:

\begin{definition}
\label{d:gphi}
 For an input formula $\phi$, Shatter generates the vertex-colored graph $G_\phi$ whose vertex set is the union of the set of clauses of $\phi$ in one color, and the set of literals of $\phi$ in another color, and whose edge set is constructed according to the following rules\footnote{Because of unit propagation, we may assume without loss of generality that every clause in $\phi$ contains at least two variables. This is implemented as a preprocessing step in Shatter.}:
\begin{itemize}
 \item For every variable $x$ in the formula, an edge is added between literals $x$ and $\overline{x}$.
 \item For every clause $C$ in $\phi$ with more than two literals, an edge is added between $C$ and every literal in $C$.
 \item For every clause $C$ in $\phi$ with exactly two literals, an edge is added between the two literals.
\end{itemize}
\end{definition}

The symmetry breaking predicates added by \toolname{Shatter} come from the automorphisms of the graph $G_\phi$. Recall the automorphisms of a graph are the permutations of its vertices that leave its edge set unchanged, and that these form a group. In the case of vertex-colored graphs, these permutations respect the coloring in the sense that no vertex is mapped to a vertex of a different color. A bird's eye view of the process \toolname{Shatter} uses to break the symmetries of a Boolean formula $\phi$ is as follows:
\begin{inparaenum}[(a)]
\item the graph $G_\phi$ is generated as per the rules above,
\item the group of automorphisms of $G_\phi$ is found,
\item a subset of these automorphisms (in particular, \toolname{Shatter} uses the generators of the group, following \cite{crawford:96}) is used to generate symmetry breaking clauses that are added to $\phi$
\end{inparaenum}.

Let $\pi$ be a permutation in the automorphism group of $G_\phi$, and consider the restriction $\pi_\mathbf{x}$ of $\pi$ to the set of literal vertices of $G_\phi$. $\pi_\mathbf{x}$ is a permutation of the literals of $\phi$ because the color of the literal vertices is different from that of the clause vertices. Let $x_1,x_2,\ldots,x_n$ be an ordering of the variables of $\phi$ and let $E_k=\left(x_1\leftrightarrow\pi_\mathbf{x}\left(x_1\right)\right)\land\cdots\land\left(x_k\leftrightarrow\pi_\mathbf{x}\left(x_k\right)\right)$ with $E_0=\top$. Then the key observation behind lexicographical symmetry breaking is that
$$\phi\in\textsc{SAT}\Leftrightarrow\phi\land\left(\bigwedge\limits_{i=0}^{n-1}E_i\rightarrow\left(x_{i+1}\rightarrow \pi_\mathbf{x}\left(x_{i+1}\right)\right)\right)\in\textsc{SAT}$$
for if $\delta$ is a satisfying assignment of $\phi$, $\delta'=\delta\circ\pi_\mathbf{x}$ is a satisfying assignment of $\phi$ as well, and one of $\delta$ and $\delta'$ will precede the other in lexicographical order (considering the strings $\delta^*\left(x_1\right)\delta^*\left(x_2\right)\ldots\delta^*\left(x_n\right)$ for $\delta^*=\delta,\delta'$) thus satisfying the additional constraints. Following the same principle, one can prove that adding lexicographical constraints for other permutations in the automorphism group of $G_\phi$ will also preserve satisfiability while reducing the number of models of the resulting formula. Notice that this formulation of symmetry breaking adds $n$ expressions involving $E_i$ subexpressions, but $E_i$ itself has a literal count that is linear in $i$, so the number of literals added per automorphism is quadratic in the number of variables of $\phi$.

This set of rules corresponds to the most basic version of lexicographical symmetry breaking. \toolname{Shatter} implements several optimizations over this basic workflow which have desirable effects on determining the satisfiability of the input formula. In Sections \ref{s:numberofmodels} and \ref{s:incompletesets}, on the other hand, we study the effect of using this tool for an \allsat{} application: reducing the number of models found through encodings of the non-arrowing property as a Boolean formula. We define the concepts from Ramsey theory that are relevant to our analysis in the following section.

\subsection{Ramsey Colorings}

For a graph $G$, let $V\left(G\right)$ and $E\left(G\right)$ be the set of vertices and edges of $G$, respectively. The order of a graph is the number of vertices of the graph. For two graphs $G$ and $H$, we denote by $\mathcal{S}\left(G,H\right)$ the set of (not necessarily induced) subgraph isomorphisms from $H$ to $G$. More formally, the set $\mathcal{S}\left(G,H\right)$ is the set of injective functions $s$ from $V\left(H\right)$ to $V\left(G\right)$ such that if $\left\{u,v\right\}\in E\left(H\right)$ then $\left\{s\left(u\right),s\left(v\right)\right\}\in E\left(G\right)$.

An edge coloring of $G$ is a function $\sigma$ from the set $E\left(G\right)$ to a finite set of colors. If the cardinality of the set of colors is $k$, the coloring is said to be a $k$-coloring. Two edge colorings of a graph $G$ are \emph{isomorphic} if one can be turned into the other by permuting the vertices of $G$, i.e., two colorings $\sigma_1$ and $\sigma_2$ are isomorphic if there exists a permutation $\pi$ of $V\left(G\right)$ such that for every edge $\left\{u,v\right\}\in E\left(G\right)$, $\sigma_1\left(\left\{\pi\left(u\right),\pi\left(v\right)\right\}\right)=\sigma_2\left(\left\{u,v\right\}\right)$. We define a ``deduplication'' operation $D$ on sets of colorings that preserves only one representative of each of the equivalence classes in the input set of colorings. Formally, for a set of colorings $A_G$ of a graph $G$, $D\left(A_G\right)$ is defined as a set of colorings such that $\sigma_1,\sigma_2\in D\left(A_G\right)$ implies $\sigma_1$ is not isomorphic to $\sigma_2$, and for every $\sigma\in A_G$ there exists a $\tilde{\sigma}\in D\left(A_G\right)$ such that $\sigma$ is isomorphic to $\tilde{\sigma}$. The choice of which representative is kept by this operation is not crucial to our application, so we can assume it is the minimum over some ordering of the colorings in $A_G$. A desirable property of $D\left(A_G\right)$ is that it can be used to recover the full set of colorings so it is an efficient way to store the entire set $A_G$.

For graphs $F$, $G$ and $H$, we say the graph $F$ arrows the pair $\left(G,H\right)$ if every 2-coloring of the edges of $F$ either contains a monochromatic $G$ in the first color or a monochromatic $H$ in the second color. We denote this property by $F\rightarrow\left(G,H\right)$. A 2-coloring of the edges of $F$ that contains no monochromatic $G$ in the first color and no monochromatic $H$ in the second color is said to \emph{witness} the fact that $F\not\rightarrow\left(G,H\right)$. The most classical example of arrowing is the fact that the edges of the complete graph on 6 vertices $K_6$ cannot be 2-colored in a way that avoids monochromatic triangles, so $K_6\rightarrow\left(K_3,K_3\right)$. The edges of $K_5$, on the other hand, can be colored as shown in Figure~\ref{f:k5example} avoiding monochromatic triangles, so $K_5\not\rightarrow\left(K_3,K_3\right)$. Several Ramsey-type problems can be expressed in terms of the arrowing property: the generalized Ramsey number $R\left(G,H\right)$ can be defined as the smallest natural number $N$ such that $K_N$ arrows the pair $\left(G,H\right)$\footnote{The ``classical'' Ramsey numbers are defined in terms of generalized Ramsey numbers as $R\left(n,m\right)=R\left(K_n,K_m\right)$ for positive integers $n$ and $m$.}, while the generalized (edge) Folkman number $F_e\left(G,H;n\right)$, which is another concept commonly studied in Ramsey theory, can be defined as the smallest order of a $K_n$-free graph that arrows the pair $\left(G,H\right)$. We define $\mathcal{C}\left(F;G,H\right)$ as the set of witness colorings of $F\not\rightarrow\left(G,H\right)$. Clearly, $\mathcal{C}\left(F;G,H\right)=\emptyset\Leftrightarrow F\rightarrow\left(G,H\right)$. We will call a set of colorings for parameters $F$, $G$ and $H$ \emph{complete}\footnote{Not to be confused with other uses of \emph{completeness} in the context of symmetry breaking for Boolean formulas. For instance, \cite{itzhakov:16} defines the concepts of \emph{sound} and \emph{complete} symmetry breaking predicates to compare the characteristics of different types symmetry breaking predicates, but in our context we have fixed the type of symmetry breaking predicate to lexicographical symmetry breaking.} if every isomorphism class in $\mathcal{C}\left(F;G,H\right)$ is represented in the set. For instance, $\mathcal{C}\left(F;G,H\right)$ itself is trivially complete, and $D\left(\mathcal{C}\left(F;G,H\right)\right)$ is complete by definition.

To use \allsat{} solvers to generate graph 2-colorings, we exploit the fact that the Boolean domain contains two values $\top$ (\emph{true}) and $\bot$ (\emph{false}) and express the negation of the arrowing property in terms of Boolean formulas. Consider the Boolean formula $\phi\left(F;G,H\right)$ on $\left|E\left(F\right)\right|$ variables $x_{\left\{u,v\right\}}$ for $\left\{u,v\right\}\in E\left(F\right)$ defined as per equation~(\ref{e:nonarrowing}) in Figure~\ref{f:nonarrowing}. This formula essentially states that in every subgraph isomorphism from $G$ to $F$ at least one of the edges involved is ``colored'' \emph{true}, and in every subgraph isomorphism from $H$ to $F$ at least one of the edges involved is ``colored'' \emph{false}. If we define the set of colors as $\left\{\bot,\top\right\}$, then a model of $\phi\left(F;G,H\right)$ can easily be mapped to an edge coloring of $F$ that avoids monochromatic copies of $G$ and $H$ in the first and second color, respectively. It is also easy to see this mapping is one-to-one, i.e., to each model of $\phi\left(F;G,H\right)$ corresponds a coloring witnessing $F\not\rightarrow\left(G,H\right)$ and vice versa. Thus, $\phi\left(F;G,H\right)\in\sat\Leftrightarrow F\not\rightarrow\left(G,H\right)$. Furthermore, one can generate $\mathcal{C}\left(F;G,H\right)$ by using an \allsat{} solver to list every model of $\phi\left(F;G,H\right)$. As mentioned before, an undesirable characteristic of this approach is that $\mathcal{C}\left(F;G,H\right)$ can be large and one is in general more interested in generating the smaller set $D\left(\mathcal{C}\left(F;G,H\right)\right)$ directly from the \sat{} formulation of $F\not\rightarrow\left(G,H\right)$.

\begin{figure*}
{\footnotesize
\begin{equation}
 \label{e:nonarrowing}
\phi\left(F;G,H\right)=\left(\bigwedge\limits_{s\in\mathcal{S}\left(F,G\right)}\bigvee\limits_{\left\{u,v\right\}\in E\left(G\right)}x_{\left\{s\left(u\right),s\left(v\right)\right\}}\right)\land\left(\bigwedge\limits_{s\in\mathcal{S}\left(F,H\right)}\bigvee\limits_{\left\{u,v\right\}\in E\left(H\right)}\overline{x_{\left\{s\left(u\right),s\left(v\right)\right\}}}\right)
\end{equation}}
\caption{The standard encoding for the non-arrowing property as a Boolean formula. Here, $\mathcal{S}\left(X,Y\right)$ denotes the set of non-induced subgraph isomorphisms from $X$ to $Y$ (see Section~\ref{s:background}).}
\label{f:nonarrowing}
\end{figure*}

\begin{figure}
\centering
 \includegraphics{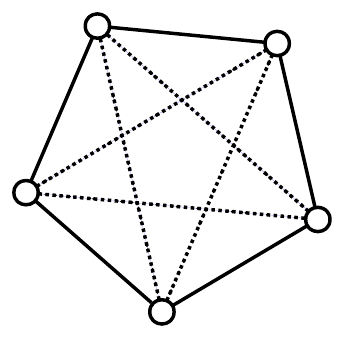}
 \caption{A coloring of $K_5$ avoiding monochromatic triangles, thus witnessing $K_5\not\rightarrow\left(K_3,K_3\right)$. The two colors are represented by solid and dashed lines.}
 \label{f:k5example}
\end{figure}

\section{Number of Satisfying Assignments}
\label{s:numberofmodels}

One of the main improvements of \toolname{Shatter} \cite{aloul:06} over the original formulation of the symmetry breaking clauses \cite{crawford:96} explained in Section~\ref{s:background} is that \toolname{Shatter} adds symmetry-breaking clauses whose number of literals is linear in the number of variables of the input formula. This is done through a relaxation on the symmetry breaking constraints. This relaxation has an undesirable effect in the number of satisfying assignments of the resulting formula. To study this effect, we summarize some of the details behind \toolname{Shatter}'s relaxation of the symmetry breaking constraints.

Using additional equality variables $e_i\equiv\left(x_i\leftrightarrow\pi_\mathbf{x}\left(x_i\right)\right)$, one incurs in a quadratic increase on the length of the formula when adding lexicographical symmetry breaking clauses, as mentioned in Section~\ref{s:background}. \toolname{Shatter} avoids this by using ``chaining predicates''. For this, new variables $l_i\equiv\left(x_i\rightarrow\pi_\mathbf{x}\left(x_i\right)\right)$ (that is, $l_i$ is true if and only if $x_i$ is ``less than or equal to'' $\pi_\mathbf{x}\left(x_i\right)$) and $p_i$ are introduced, together with clauses $p_i\leftrightarrow\left(e_{i-1}\rightarrow\left(l_i\land p_{i+1}\right)\right)$, with $e_0=\top$ and $p_{n+1}=\top$. \toolname{Shatter} also replaces equality variables $e_i$ with ``greater than or equal to'' variables $g_i\equiv\left(\pi_\mathbf{x}\left(x_i\right)\rightarrow x_i\right)$, and relaxes the if and only if clauses for the $p_i$ variables to one way implications. The clauses added by \toolname{Shatter} are then the CNF equivalents of formulas of the form $p_i\rightarrow\left(g_{i-1}\rightarrow l_i\land p_{i+1}\right)$, with $g_0=\top$. It is easy to see these relaxations introduce satisfying assignments that do not satisfy the original symmetry breaking formulation. To get a better feel of how much of a ``blow-up'' in the number of satisfying assignments does this relaxation cause, consider the following lemmas about extensions of partial assignments of the Boolean expressions involved in these two formulations of symmetry breaking clauses. $S_k$ in Lemma~\ref{l:originalmodelcount} corresponds to the original lexicographical symmetry breaking clauses with chaining, while $T_k$ in Lemma~\ref{l:shattermodelcount} corresponds to the clauses added by Shatter.

\begin{lemma}
\label{l:originalmodelcount}
 Let $S_k=e_0\land\bigwedge\limits_{i=1}^{k}\left(p_i\leftrightarrow\left(e_{i-1}\rightarrow\left(l_i\land p_{i+1}\right)\right)\right)\land p_{k+1}$. Then a partial assignment of the variables $e_i$ and $l_i$ has at most one extension that satisfies $S_k$.
\end{lemma}

\begin{proof}
Let $S'_k=e_0\land\bigwedge\limits_{i=1}^{k}\left(p_i\leftrightarrow\left(e_{i-1}\rightarrow\left(l_i\land p_{i+1}\right)\right)\right)$, that is, $S'_k$ is $S_k$ without fixing $p_{k+1}$. Then $S_k=S'_k\land p_{k+1}$. We first prove by induction in $k$ that a partial assignment of the variables $e_i$ and $l_i$ has at most one extension that satisfies $S_k'$ once $p_{k+1}$ is assigned to a value. For $k=0$, this is trivially true since the value of $S'_0=e_0$ does not depend on the value of $p_1$. We now assume that this property holds for $S'_n$ and prove it for $S'_{n+1}$. Our induction hypothesis implies that for any assignment of the variables $e_i$ and $l_i$ for $i\leq n$, there is at most one assignment of the variables $p_i$, $i\leq n$ that satisfies $S'_n$ once $p_{n+1}$ is assigned to a value. Suppose we assign $p_{n+2}$ to a value, then since $p_{n+1}\leftrightarrow \left(e_{n-1}\rightarrow l_n\land p_{n+1}\right)$ is a subformula of $S'_{n+1}$ so $p_{n+1}$ is uniquely determined by the assignments of $e_{n-1}$, $l_n$ and $p_{n+2}$, which proves our claim. Since $S_k=S_k'\land p_{k+1}$, any satisfying assignment of $S_k$ must fix $p_{n+2}$ to $\top$ and the lemma follows.
\end{proof}

\begin{lemma}
\label{l:shattermodelcount}
 Let $T_k=g_0\land\bigwedge\limits_{i=1}^{k}\left(p_i\rightarrow\left(g_{i-1}\rightarrow\left(l_i\land p_{i+1}\right)\right)\right)\land p_{k+1}$. Let $\phi$ be a partial assignment that assigns the $m$ variables $g_i,g_{i+1},\ldots,g_{i+m-1}$ to $\bot$. Then, if $\phi$ can be extended to an assignment that satisfies $T_k$, it can be extended to at least $2^{m}-1$ assignments that satisfy $T_k$.
\end{lemma}

\begin{proof}
Since $\phi$ assigns $g_i$ to $\bot$, the subformula $p_{i+1}\rightarrow\left(g_i\rightarrow\left(l_{i+1}\land p_{i+2}\right)\right)$ simplifies to $p_{i+1}\rightarrow\top$, regardless of what $l_{i+1}$ and $p_{i+2}$ are assigned to. Then $p_i$ can be assigned to $\bot$ or $\top$ without falsifying $T_k$. Furthermore, $p_{i+1}\rightarrow\left(g_i\rightarrow\left(l_{i+1}\land p_{i+2}\right)\right)$ does not constrain $p_{i+2}$ if $g_i$ is assigned to $\bot$, and $p_{i+2}\rightarrow\left(g_{i+1}\rightarrow\left(l_{i+2}\land p_{i+3}\right)\right)$ simplifies to $p_{i+2}\rightarrow\top$ when $g_{i+1}$ is assigned to $\bot$, so $p_{i+2}$ can be assigned to $\bot$ or $\top$ as well. Following the same reasoning, all the variables $p_{i+1},p_{i+2},\ldots,p_{i+m}$ can be assigned to $\bot$ or $\top$ independently without falsifying $T_k$ except possibly for $p_{i+m}$ if $i+m=k+1$, yielding the desired result.
\end{proof}

The combination of these lemmas means that, in the worst case, the increase in the number of satisfying assignments due to Shatter's relaxations may be exponential in the number of $e_i$ or $g_i$ variables. While in the original lexicographical symmetry breaking formulation the number of $e_i$ variables is the same as the number of variables in the original formula, \toolname{Shatter} reduces the number of variables to consider by eliminating certain tautological subformulas from $T_k$, thus mitigating the effect of the increase in the number of models.

Another important detail in the analysis above is that the ``blow-up'' in the number of models of the formula enhanced with symmetry breaking clauses happens only at the $p_i$ variables: for a fixed assignment of the original variables of $\phi$ there are many possible assignments of the $p_i$ variables that would satisfy the formula output by Shatter. The projective model enumeration technique \cite{gebser:09} in the \toolname{clasp} solver \cite{gebser:12} is particularly useful to nullify this increase in the number of models. This technique allows for outputting models that are different modulo a subset of the variables. Thus combining projective model enumeration with the output of the formula preprocessed by \toolname{Shatter} restores the original goal of symmetry breaking which is reducing the number of satisfying assignments. The subset of variables that one would project to would of course be the original set of variables of the input formula.

In Section \ref{s:conclusions} we point at \toolname{BreakID} \cite{devriendt:16a}, another symmetry breaking tool, which avoids this issue altogether by using the original symmetry breaking predicates without relaxations. This tool is thus more appropriate for \allsat{} applications where symmetry breaking is needed.

\subsection{An Example: $\phi\left(K_8;C_5,C_5\right)$}
\label{s:example}

To illustrate this issue we provide a concrete example. From finite Ramsey theory, we know that $R\left(C_5,C_5\right)=9$ \cite{chartrand:71}, where $C_5$ is the cycle of length 5 (see also \cite{radziszowski:94} for a comprehensive survey of what is known in finite Ramsey theory). This means that $\phi\left(K_9;C_5,C_5\right)\notin\sat$, but $\phi\left(K_8;C_5,C_5\right)\in\sat$, so we are interested in finding all edge colorings of the complete graph $K_8$ witnessing $R\left(C_5,C_5\right)>8$. $\phi\left(K_8;C_5,C_5\right)$ contains 28 variables (corresponding to $\genfrac{(}{)}{0pt}{}{8}{2}$ edges in $K_8$) and 1344 clauses, and there are 1190 models for that formula. From this information, we know that $\left|\mathcal{C}\left(K_8;C_5,C_5\right)\right|=1190$. After processing this formula with Shatter, the resulting formula with symmetry breaking clauses has 70 variables, 1499 clauses, and 824 models. On the other hand, using our own implementation of the chaining method without the relaxation outputs a formula with 165 variables, 1809 clauses, and 5 models. Using \toolname{nauty} \cite{mckay:14} to reduce any of these sets of colorings to pick just one representative from each equivalence class of colorings under isomorphism, we find that $\left|D\left(\mathcal{C}\left(K_8;C_5,C_5\right)\right)\right|=4$, so the chaining method without the relaxation outputs only one redundant coloring.

\section{Incomplete Sets of Colorings}
\label{s:incompletesets}

Perhaps more concerning than the increase in the number of colorings found by preprocessing the formula $\phi\left(F;G,H\right)$ is the fact that enumerating all models of the result of preprocessing $\phi\left(F;G,H\right)$ with \toolname{Shatter} may not yield a complete (according to our definition in Section~\ref{s:background}) set of colorings. On the other hand, there are parameters $F$, $G$, and $H$ for which preprocessing a formula $\phi\left(F;G,H\right)$ with \toolname{Shatter} does output a Boolean formula whose models can be used to build a complete set of colorings for the given parameters. One example was presented in Section~\ref{s:example}, where we were able to generate $D\left(\mathcal{C}\left(K_8;C_5,C_5\right)\right)$ from the models of $\phi\left(K_8;C_5,C_5\right)$ after preprocessing it with Shatter.

The fact that for certain parameters this method will work and for other parameters it will not warrants an investigation. In this section, we take a closer look at this phenomenon and present a sufficient condition under which a workflow that uses \toolname{Shatter} for symmetry breaking produces incomplete sets of colorings. This condition is the presence of \emph{free variables}, which are variables that do not appear in the CNF formula. This is possible in the context of CNF formulas in the DIMACS format because the specification\footnote{\url{ftp://dimacs.rutgers.edu/pub/challenge/satisfiability/doc/satformat.dvi}} states that ``it is not necessary that every variable appear in the instance.'' We study the effect of free variables in Lemma~\ref{l:freevars} and as an immediate consequence, we state a sufficient condition for $\phi\left(F;G,H\right)$ to produce incomplete sets of colorings in Theorem~\ref{t:incomplete}.

\begin{lemma}
 \label{l:freevars}
 Let $\phi$ be a Boolean formula. Then any satisfying assignment of the formula output by preprocessing $\phi$ with \toolname{Shatter} will assign any free variables in $\phi$ to $\bot$.
\end{lemma}

\begin{proof}
 Let $W$ be the set of free variables in $\phi$. The graph $G_\phi$ generated according to the rules summarized in Section~\ref{s:background} has 2 vertices for each free variable (one for the positive literal and one for the negative literal) and an edge between these two vertices. Then the matching on $2\left|W\right|$ vertices (i.e., the disjoint union of $\left|W\right|$ edges) is an induced subgraph of $G_\phi$. From graph theory, we know that the permutations $\left(x,\overline{x}\right)$ for $x\in W$ are generators of the automorphism group of $G_\phi$ (see, for instance, \cite{white:73}, Theorem 3-11). The rules applied by \toolname{Shatter} will turn these permutations into clauses of the form $\overline{x}\lor\overline{x}$ for $x\in W$, so each variable in $W$ will be assigned to $\bot$.
\end{proof}

The above proof is based on the rules presented in \cite{aloul:06} to generate the symmetry breaking clauses. In practice, version 0.3 of \toolname{Shatter}\footnote{\url{http://www.aloul.net/Tools/shatter/Shatter_Linux_v03.tar.gz}} seems to have an additional rule that assigns each free variable to its own color. This change is effective in eliminating permutations of the type $\left(x_i,x_j\right)\left(\overline{x_i},\overline{x_j}\right)$ for $x_i,x_j\in W$ from the set of generators of the automorphism group of $G_\phi$ (because $x_i$ and $x_j$ will now have different colors), yet it does not eliminate permutations of the type $\left(x,\overline{x}\right)$ which are the culprit of Lemma~\ref{l:freevars}. It is important to highlight the fact that the addition of single-clause variables (a clause $\overline{x}\lor\overline{x}$ is equivalent to $\overline{x}$) is not a flaw in the design of the symmetry breaking clauses in \toolname{Shatter} but a feature, as noted in Section 3  of \cite{aloul:02}. This is because processing permutations this way reflects the fact that free variables in the input formula can be fixed to any Boolean value, in this case $\bot$, which is advantageous when trying to determine the satisfiability of the input formula.

As an immediate consequence of Lemma~\ref{l:freevars}, we have the following theorem.

\begin{theorem}
\label{t:incomplete}
 Let $F$, $G$ and $H$ be graphs and suppose an edge $e\in E\left(F\right)$ does not participate in any subgraph isomorphism from $G$ to $F$ or from $H$ to $F$ (i.e., $e\notin\image\left(\sigma\right)$ for any $\sigma\in\mathcal{S}\left(F,G\right)\cup\mathcal{S}\left(F,H\right)$). Then, if $\mathcal{C}\left(F;G,H\right)$ is not empty, the set of colorings obtained from the models of $\phi\left(F;G,H\right)$ after preprocessing it with \toolname{Shatter} is incomplete.
\end{theorem}

\begin{proof}
 Let $\phi'\left(F;G,H\right)$ be the formula output by preprocessing $\phi\left(F;G,H\right)$ with Shatter. Because $e$ does not participate in any subgraph isomorphism from $G$ to $F$, or from $H$ to $F$, the variable $x_{e}$ is free in $\phi\left(F;G,H\right)$. By Lemma~\ref{l:freevars} it will be assigned to $\bot$ in any model of $\phi'\left(F;G,H\right)$. Let $m$ be a model of the formula $\phi'\left(F;G,H\right)$ and let $m'$ be the assignment obtained from $m$ by assigning $x_{e}$ to $\top$. Because $e$ is free in $\phi\left(F;G,H\right)$, the restriction of $m'$ to the variables of $\phi\left(F;G,H\right)$ (i.e., the restriction of $m'$ to the original variables) is a model of $\phi\left(F;G,H\right)$, but $m'$ itself is not a model $\phi'\left(F;G,H\right)$ thus the coloring corresponding to $m'$ will not be represented in the colorings obtained from the models of $\phi'\left(F;G,H\right)$.
\end{proof}

An obvious modification one can do to avoid this issue is to restrict the set of variables in equation~\ref{e:nonarrowing} in Figure~\ref{f:nonarrowing} to $x_e$ for edges $e\in E\left(F\right)$ that are actually involved in some subgraph isomorphism from $G$ to $F$ or from $H$ to $F$. Unfortunately, this does not guarantee the resulting formula will output a complete set of colorings. There are cases where the models of formulas encoding the non-arrowing property do not generate complete sets of colorings even when the formula does not have free variables. This indicates Theorem~\ref{t:incomplete} is not a necessary condition for this phenomenon to occur. Here we provide an example of such a formula, which is related to the example presented in Section \ref{s:example}: Let $K_{ex}$ be the graph obtained from $K_8$ by selecting one of its vertices and removing all but 2 edges incident to it (alternatively, this graph can be obtained by adding a vertex to $K_7$ and connecting it to two of the original vertices). The resulting graph is illustrated in Figure~\ref{f:notnecessary}. We are interested in obtaining all 2-colorings avoiding monochromatic $C_5$ in any of the colors. It is easy to see that $\phi\left(K_{ex};C_5,C_5\right)$ has no free variables: because the only vertex of degree less than 6 has degree 2, it participates in at least one cycle of length 5. Nevertheless, the models of the formula obtained from preprocessing $\phi\left(K_{ex},C_5,C_5\right)$ with \toolname{Shatter} represent 64 of the 90 possible isomorphism classes in $\mathcal{C}\left(K_{ex};C_5,C_5\right)$.

\begin{figure}
 \centering
 \includegraphics[width=4cm]{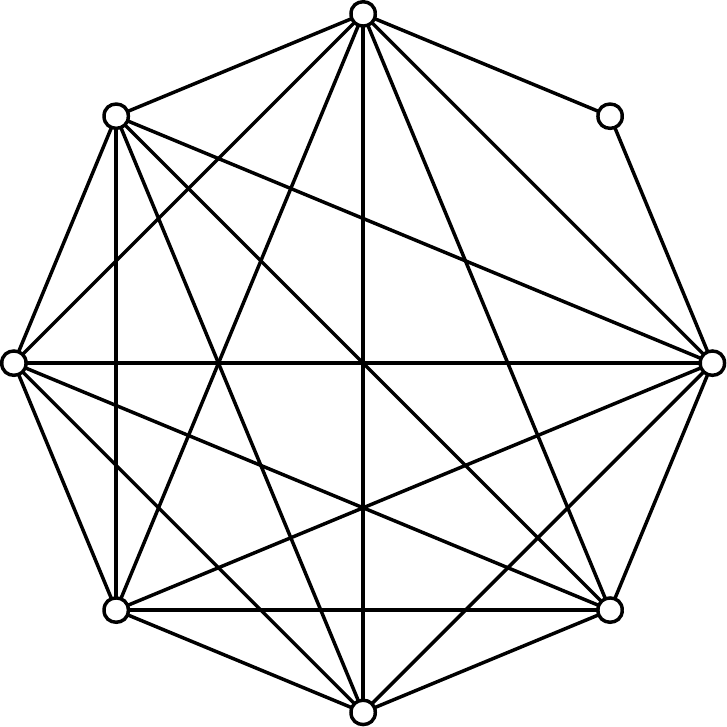}
 \caption{An illustration of the graph $K_{ex}$ in Section \ref{s:incompletesets}. The models of the formula output by \toolname{Shatter} on input $\phi\left(K_{ex};C_5,C_5\right)$ cannot be used to create a complete set of colorings witnessing $K_{ex}\not\rightarrow\left(C_5,C_5\right)$.}
 \label{f:notnecessary}
\end{figure}

\section{Conclusion and Related Work}
\label{s:conclusions}

We looked at the use of \toolname{Shatter}, a popular tool for symmetry breaking of CNF formulas, in the context of \allsat{} problems. Specifically, we presented and discussed two issues related to applying \toolname{Shatter} to formulas encoding the non-arrowing property. The first issue was related to the number of satisfying assignments, showing that \toolname{Shatter} may incur in an increase in the number of models found. We also mentioned the projective enumeration feature in \toolname{clasp} as a tool to deal with this issue. The second issue we discussed was related to the completeness of the set of colorings generated from the Boolean formula after preprocessing with Shatter. We presented a sufficient condition for the set of colorings obtained this way to be incomplete. A direction for future work in this area is completing this analysis by providing necessary conditions.

While \toolname{Shatter} has been an influential tool in the field of symmetry breaking in Boolean formulas for over a decade, it is not the only tool available as a drop-in addition to existing workflows. Recently, \toolname{BreakID} \cite{devriendt:16a} has built upon the symmetry breaking techniques introduced by \toolname{Shatter} and has added some novel ideas like row interchangeability. Even though \toolname{BreakID} implements the same relaxations as \toolname{Shatter}, it does includes an option to no use these relaxations and is thus better suited for \allsat{} applications since it will not introduce additional models.

\section*{Acknowledgments}

The author would like to thank Dr. Edith Hemaspaandra and Dr. Stanisław P. Radziszowski at the Rochester Institute of Technology, as well as the anonymous reviewers of the Thirty-Second AAAI Conference on Artificial Intelligence (AAAI-18) for their valuable comments and suggestions for improvements.

\bibliographystyle{abbrv}

\end{document}